\documentclass[lettersize,journal]{IEEEtran}
\usepackage{amsmath,amsfonts}
\usepackage{amsthm}
\usepackage{algorithmic}
\usepackage{algorithm}
\usepackage{array}
\usepackage[caption=false,font=normalsize,labelfont=sf,textfont=sf]{subfig}
\usepackage{textcomp}
\usepackage{stfloats}
\usepackage{url}
\usepackage{verbatim}
\usepackage{graphicx}
\usepackage{cite}
\usepackage{booktabs}
\usepackage{thmtools}
\usepackage{thm-restate}
\newtheorem{remark}{Remark}%

\declaretheorem[name=Lemma]{lemma}

\allowdisplaybreaks[4]

\begin{document}

\title{Policy Optimization with Smooth Guidance Learned from State-Only Demonstrations}

\author{Guojian~Wang~\IEEEmembership{Graduate Student Member,~IEEE},
        Faguo~Wu,
        Xiao~Zhang,
        and~Tianyuan~Chen
\thanks{Manuscript received June 14, 2024.}
\thanks{Xiao Zhang and Faguo Wu are the corresponding authors (e-mail: xiao.zh@buaa.edu.cn, faguo@buaa.edu.cn).}
\thanks{Guojian Wang, Xiao Zhang are with the School of Mathematical Sciences, Beihang University, Beijing 100191, China, and with the Key Laboratory of Mathematics, Informatics, and Behavioral Semantics, Ministry of Education, Beijing 100191, China, and also Beijing Advanced Innovation Center for Future Blockchain and Privacy Computing, Beijing 100191, China. Xiao Zhang is also with Zhongguancun Laboratory, Beijing 100194, China (e-mails: wgj@buaa.edu.cn and xiao.zh@buaa.edu.cn).}
\thanks{Faguo Wu and Tianyuan Chen are with the Institute of Artificial Intelligence, Beihang University, Beijing 100191, China, and with the Key Laboratory of Mathematics, Informatics, and Behavioral Semantics, Ministry of Education, Beijing 100191, China, and with Beijing Advanced Innovation Center for Future Blockchain and Privacy Computing, Beijing 100191, China, and also with Zhongguancun Laboratory, Beijing 100194, China (e-mails: faguo@buaa.edu.cn and ctymath@buaa.edu.cn).}
}

\markboth{Journal of \LaTeX\ Class Files,~Vol.~14, No.~8, August~2021}%
{Shell \MakeLowercase{\textit{et al.}}: A Sample Article Using IEEEtran.cls for IEEE Journals}


\maketitle

\begin{abstract}
The sparsity of reward feedback remains a challenging problem in online deep reinforcement learning (DRL). Previous approaches have utilized offline demonstrations to achieve impressive results in multiple hard tasks. However, these approaches place high demands on demonstration quality, and obtaining expert-like actions is often costly and unrealistic. To tackle these problems, we propose a simple and efficient algorithm called Policy Optimization with Smooth Guidance (POSG), which leverages a small set of state-only demonstrations (where expert action information is not included in demonstrations) to indirectly make approximate and feasible long-term credit assignments and facilitate exploration. Specifically, we first design a trajectory-importance evaluation mechanism to determine the quality of the current trajectory against demonstrations. Then, we introduce a guidance reward computation technology based on trajectory importance to measure the impact of each state-action pair, fusing the demonstrator's state distribution with reward information into the guidance reward. We theoretically analyze the performance improvement caused by smooth guidance rewards and derive a new worst-case lower bound on the performance improvement. Extensive results demonstrate POSG's significant advantages in control performance and convergence speed in four sparse-reward environments, including the grid-world maze, Hopper-v4, HalfCheetah-v4, and Ant maze. Notably, the specific metrics and quantifiable results are investigated to demonstrate the superiority of POSG.
\end{abstract}


\section{Introduction}
\IEEEPARstart{I}{n} recent years, deep reinforcement learning (RL) has demonstrated remarkable accomplishments in tackling sequential decision-making challenges across diverse domains, including the Arcade Learning Environment~\cite{Mnih2015HumanlevelCT,bellemare2013arcade,ecoffet2021first}, the game of Go~\cite{Silver2016MasteringTG}, continuous locomotive control~\cite{Lillicrap2016ContinuousCW, Schulman2015TrustRP,fujimoto2018addressing}, and robotic navigation~\cite{Florensa2017StochasticNN,zhanghierarchical,li2022active}. Despite these celebrated achievements, reinforcement learning remains a formidable task in scenarios characterized by sparse or delayed rewards~\cite{yang2021exploration,gangwani2020learning}, primarily due to the challenge of striking a balance between exploration and exploitation in environments with sparse or delayed rewards~\cite{guo2020memory}. Conventional deep RL algorithms confront exploration difficulties, as fully exploring the entire state-action space is often unfeasible and cannot be spontaneously guaranteed, particularly in settings with sparse environmental rewards~\cite{jing2020reinforcement}.

One possible solution to the problem of sparse or delayed rewards is temporal credit assignment. Credit assignment aims to understand the relevance between actions and outcomes and measure the impact of actions the agent performs on future rewards. Some recent studies propose building an environmental model to obtain a more fine-grained description of the action's effect~\cite{Ha2018WorldM, Kaiser2019ModelBasedRL, Schrittwieser2019MasteringAG}. However, this approach increases the computational burden, and acquiring an accurate model in complex and partially observed environments remains difficult. On the other hand, model-free credit assignment methods leverage hindsight information~\cite{harutyunyan2019hindsight}, counterfactual~\cite{mesnard2021counterfactual}, episode memory~\cite{Hung2018OptimizingAB}, transformers~\cite{Ferret2019CreditAA}, and return decomposition~\cite{arjona2019rudder} to perform long-term credit assignments. However, the premise of these methods is to obtain good trajectories with sparse rewards, and these designated methods might cause the problem of unstable training and parameter sensibility.

Many research advances have assisted policy exploration of agents by learning from demonstrations (LfD)~\cite{zhou2020watch}. An intuitive LfD approach enhances RL by data augmentations, which maintains the expert demonstrations in a replay buffer for value estimation~\cite{hester2018deep,vecerik2017leveraging,gulcehre2019making,libardi2021guided}. Some LfD methods utilize demonstrations to pre-train the policy by supervised learning~\cite{silver2016mastering,cruz2017pre}. These algorithms force the agent's policy to conform to the expert's policy and do not reuse them during the policy optimization procedure. Recent LfD studies draw inspiration from imitation learning and encourage the agent to mimic the demonstrated actions~\cite{kang2018policy,Sun2018TruncatedHP,jing2020reinforcement,rengarajan2022reinforcement}. Specifically, these methods either augment the original RL loss function with a divergence regularization term or design a new shaping reward derived from a distribution divergence function to force expert-alike exploration. 


In summary, the limitations of both CA and LfD approaches impede their practicality in sparse-reward settings. Firstly, it can be troublesome for CA methods to obtain highly rewarded trajectories in sparse-reward environments with large state spaces. Moreover, CA methods pose a computational burden and face challenges in accurately estimating the influence of individual state-action pairs. Secondly, most LfD methods demand high-quality samples and rely on flawless and sufficient demonstrations, including complete state-action information. Our study lies at the intersection of CA and LfD methods, aiming to mitigate their shortcomings. This paper introduces a straightforward and practical reinforcement learning approach named Policy Optimization with Smooth Guidance (POSG), designed for seamless integration into existing RL algorithms. Our method facilitates approximate long-term credit assignments and extends the applicability of CA approaches. Furthermore, this method addresses the limitations of LfD approaches by introducing a dense demonstration-based guidance reward function, thereby reducing the requirements for LfD demonstrations.

The proposed approach leverages only a few, or even a single state-only demonstration trajectory, to indirectly estimate the impact of each state-action pair. The fundamental insight lies in utilizing state-only demonstration trajectories to assess the impact of current state-action pairs, thereby amalgamating state distribution information from demonstrations with their associated reward signals. We introduce two technologies to compute the impact of each state-action pair: a trajectory-importance evaluation mechanism and a smooth guidance reward computation technology. Specifically, the trajectory-importance evaluation mechanism estimates trajectory importance based on the maximum mean discrepancy (MMD) distance to demonstrations and the returns of corresponding trajectories. Subsequently, the guidance reward for each state-action pair is derived through a smooth weighted average of trajectory importance. Furthermore, we establish a novel worst-case lower bound for POSG policy optimization and offer a theoretical guarantee of performance improvement. Extensive experimental evaluations demonstrate that POSG surpasses other baseline algorithms across a discrete grid-world maze and three continuous locomotion control tasks.


Our main contributions are summarized as follows:
\begin{enumerate}
  \item To avoid the CA's dilemma, we propose an RL method that utilizes offline state-only demonstrations to achieve approximate credit assignments in sparse-reward settings.
  \item POSG can use a single state-only demonstration trajectory to solve RL's exploration problem and improve RL's sample efficiency, demonstrated experimentally in Section~\ref{sec:abla_amount}.
  \item No additional neural networks are needed to train. Our algorithm is simple in form and explicit in physical meaning, which fuses the distribution information of demonstrations and the return signals of relevant trajectories.
  \item A new worst-case lower bound is deduced to provide a performance improvement guarantee for POSG.
  \item This study discusses the superior performance of POSG over other state-of-the-art RL algorithms across a discrete grid-world maze and three continuous locomotion control tasks.  Significantly, the analysis focuses on specific metrics and quantifiable results to illustrate the superiority of POSG.
\end{enumerate}

The rest of this article is organized as follows. Section~\ref{sec:related} describes some important progress of the recent related work. Section~\ref{sec:backgroung} briefly introduces the related preliminary knowledge. Then, the method we propose is introduced in detail in Section~\ref{sec:approach}. Section~\ref{sec:thm} describes the theoretical analyses of POSG. The environmental settings are introduced in Section~\ref{sec:setup}. In Section~\ref{sec:experience}, we experimentally demonstrate the feasibility and effectiveness of the proposed algorithm in terms of exploration efficiency and learning speed. Finally, we summarize the main work of this study in Section~\ref{sec:conclusion}.

\section{Related work}
\label{sec:related}

In this section, we summarize recent work related to this research.

\textbf{Learning from Demonstrations.} 
LfD combines RL with expert demonstration data to improve policy exploration and accelerate learning. LfD trains a behavioral policy with a state-action visitation distribution similar to the demonstrator's to accomplish this goal. Early LfD algorithms sample replay buffers containing expert demonstrations and self-generated data simultaneously to enhance the learning ability of the agent~\cite{hester2018deep,vecerik2017leveraging}. Self-imitation learning (SIL) methods~\cite{oh2018self,gangwani2019learning,libardi2021guided} train the agent to imitate its own past experiences only when the return of the previous episode is greater than the value estimate of the agent or returns of trajectories in the replay buffer. Episodic reinforcement learning methods utilize episodic memories to estimate the value of the state precisely and propagate its value to the previous states~\cite{pritzel2017neural,blundell2016model,lin2018episodic,lee2019sample,hu2021generalizable,le2021model}. Many other previous works introduce novel methods that enable the agent learns a range of diverse exploratory policies based on episodic memory~\cite{guo2020memory,badia2020never}. This study uses demonstration experiences in the state-action space to construct a shaping function, which can easily be integrated with existing RL methods, such as PPO~\cite{schulman2017proximal}.

\textbf{State-Only Imitation Learning (IL).}
Imitation learning aims to learn a control policy that imitates the behaviors of experts and outputs the same action when the agent receives the observation that occurred in the demonstration data set. The state-only IL method is a branch of imitation learning where the requirement for the demonstrator's action information is alleviated. This approach expands the scope of the realistic application of IL. GAIfO~\cite{torabi2018generative} proposes to learn a policy to output actions that lead to similar effects as demonstrations.  I2L~\cite{Gangwani2020StateonlyIW} and SAIL~\cite{Liu2019StateAI} train the agent by minimizing the distance between state visitation distributions of the current policy and demonstrations. AILO~\cite{Gangwani2022ImitationLF} leverages demonstrations only containing observations to train an intermediary policy whose state transitions are close to the expert dataset. Our method can only utilize a few state-only demonstrations to achieve approximate credit assignments and avoid introducing extra neural networks.

\textbf{Metrics to Compute the Distance between Policies or Trajectories.}
Many studies propose to compute the distance between policies using Kullback–Leibler (KL), Bregman, or $f$-divergence~\cite{lee2019efficient,ghasemipour2020divergence,wang2019divergence}. R\'{e}nyi divergence is introduced to compute the discrepancy of state visitation between different trajectories~\cite{Yuan2022RewardingEV}; however, this divergence is a parametric distance measure, and this approach employs $k$-NN estimator to estimate the parameters of distributions efficiently. MADE maximizes the derivation of state-action visitations of the current policy from the explored regions of previous trajectories~\cite{zhang2021made}. However, MADE relies on estimating the state-action visitation density, which can be non-trivial in high-dimensional control tasks. Similar to our method, CQL~\cite{kumar2020conservative} and MCPO~\cite{jing2020reinforcement} adopt the MMD metric as the discrepancy measure; however, these algorithms require both state and action information of demonstrations to update the policy parameter.

\textbf{Credit Assignment.}
Various works have focused on the credit assignment problem, and credit assignment methods can be integrated with existing RL algorithms easily~\cite{pocius2018comparing}. These methods assist us in understanding the association between sparse or delayed rewards and state or state-action pairs and reduce learning time by providing dense and supplemental rewards~\cite{ng1999policy}. Zheng et al.~\cite{zheng2022adaptive} formulate pairwise weight functions of the state where the action is performed, the future highly-rewarded state, and the time horizon between the two states, which is learned by a special meta-gradient. Hindsight Credit Assignment (HCA)~\cite{harutyunyan2019hindsight} converts the credit assignment problem into a supervised learning task by learning a hindsight probability function of actions. Counterfactual Credit Assignment (CCA)~\cite{mesnard2021counterfactual} uses a value function baseline with a hindsight information vector to implement credit assignment implicitly while avoiding giving away information about the agent's actions to reduce potential bias. In RUDDER~\cite{arjona2019rudder}, an LSTM network is used to learn to redistribute the return of a trajectory to the preceding states before the final rewarding state. Episodic Backward Update~\cite{lee2019sample} and Neural Episodic Control~\cite{pritzel2017neural} enable efficient reward propagation by sampling from episodic memories and updating the value of all transitions more quickly. State Associative Learning~\cite{raposo2021synthetic} propagates credit directly by learning associations between states and arbitrary future states. Xu et al.~\cite{xu2018meta} introduce an online meta-learning method to learn hyper-parameters of a discount $\gamma$ and bootstrapping parameter $\lambda$ of Temporal Difference (TD) for credit assignment. In contrast, our method does not incur high computational costs by dispensing with training auxiliary networks and can be regarded as a simple weighted return decomposition. 

\section{Preliminaries}
\label{sec:backgroung}
This paper studies the credit assignment problem of reinforcement learning in tasks with sparse or delayed rewards. Before describing our method in detail, we introduce some preliminary knowledge about RL in this section. Then, the definition of Maximum Mean Discrepancy (MMD), which plays a vital role in our method, is provided.

\subsection{Reinforcement Learning} 
A typical RL problem is modeled as an infinite-horizon Markov decision process with discrete time, which can be defined as a tuple $ M = (\mathcal{S}, \mathcal{A}, P, r, \rho_{0}, \gamma)$. Here, $\mathcal{S}$ is a discrete or continuous state space, $\mathcal{A}$ denotes a discrete (or continuous) action space, $P: \mathcal{S}\times\mathcal{A}\rightarrow\Pi(\mathcal{S})$ is the transition probability distribution, where $\Pi(\mathcal{S})$ is the space of probability distributions over the state $\mathcal{S}$. In addition, $r_e: \mathcal{S}\times\mathcal{A}\rightarrow[R_{min}, R_{max}]$ is the environmental reward function, in which we assume that the minimum and maximum value of the reward function is $R_{min}$ and $R_{max}$, respectively. Furthermore, $\rho_{0}$ is the initial state distribution, and $\gamma\in [0,1]$ is a discount factor. A stochastic policy $\pi_{\theta}: \mathcal{S}\rightarrow\mathcal{P}(\mathcal{A})$ parameterized by $\theta$, maps the state space $ \mathcal{S} $ to the set of probability distributions over the action space $\mathcal{A}$. Generally, the optimization objective of RL is to find a policy $\pi_{\theta}$ that maximizes the expected discounted return:
\begin{equation}
    \label{eq:rl_objective}
    \eta(\pi_{\theta}) = \mathbb{E}_{s_0,a_0,\dots}\left[\sum_{t=0}^{\infty}\gamma^{t}r_e(s_{t},a_t)\right],
\end{equation}
where $s_0\sim \rho_0(s_0)$, $a_t\sim\pi_{\theta}(a_t \vert s_t)$, and $s_{t+1}\sim P(s_{t+1}\vert s_t,a_t)$. Meanwhile, we use the following standard definitions of the state-action value function $Q_\pi$ and value function $V_\pi$:
\begin{equation}
  \label{eq:Q_func}
  Q_\pi(s_t, a_t)=\mathbb{E}_{s_{t+1}, a_{t+1},\cdots}\left[\sum_{l=0}^{\infty}\gamma^{l} r_e(s_{t+l},a_{t+l})\right],
\end{equation}
and
\begin{equation}
  \label{eq:V_func}
  V_\pi(s_t)=\mathbb{E}_{a_t, s_{t+1}, a_{t+1},\cdots}\left[\sum_{l=0}^{\infty}\gamma^{l} r_e(s_{t+l},a_{t+l})\right],
\end{equation}
where $a_t\sim\pi_{\theta}(a_t \vert s_t)$, and $s_{t+1}\sim P(s_{t+1}\vert s_t,a_t)$. Then, the advantage function is expressed as:
\begin{equation}
  \label{eq:adv_func}
  A_\pi(s_t, a_t) = Q_\pi(s_t, a_t) - V_\pi(s_t).
\end{equation}

When $\gamma<1$, the discounted state visitation distribution $d_\pi$ is given by: $d_\pi(s) = (1-\gamma)\sum_{t=0}^{\infty}\gamma^t \mathbb{P}(s_t=s\vert \pi)$, where $\mathbb{P}(s_t=s\vert \pi)$ denotes the probability of $s_t=s$ with respect to the randomness induced by $\pi$, $P$ and $\rho_0$.

\subsection{Maximum Mean Discrepancy}
\label{sec:mmd_preliminaries}
In this paper, we treat trajectories as deterministic policy distributions and use the maximum mean discrepancy (MMD) to empirically measure the difference (or similarity) between trajectories~\cite{gretton2006kernel, Gretton2012OptimalKC,masood2019diversity}. MMD is a non-parametric test statistic, and different from Kullback-Leibler (KL), Jensen–Shannon (JS), or $f$-divergences, calculating the MMD metric does not require the distribution parameters. This property makes it suitable for the problem considered in this study, where trajectories are treated as deterministic policies. The definition of MMD depends on the choice of function space.

Assume that the probability distributions $p$ and $q$ are defined on the space $\mathbb{X}$, and let $x$ and $y$ be the elements sampled from $p$ and $q$, respectively. Given a reproducing kernel Hilbert space (RKHS) $\mathcal{H}$ with the kernel function $k(\cdot,\cdot)$; we define the MMD as follows:
\begin{equation}
  \label{eq:MMD_rkhs}
  \begin{aligned}
    \mathrm{MMD}^2(p, q, \mathcal{H}) = \mathbb{E}[k(x, x^\prime)] - 2\mathbb{E}[k(x, y)] + \mathbb{E}[k(y, y^\prime)],
  \end{aligned}
\end{equation}
where $x, x^\prime\ \mathrm{i.i.d.}\sim p$ and $y, y^\prime\ \mathrm{i.i.d.} \sim q$. Using Eq.~\eqref{eq:MMD_rkhs} to estimate the distance between $p$ and $q$ is tractable because of the inherent property of RKHS~\cite{gretton2006kernel,Gretton2012OptimalKC}.

\subsection{Trajectory-Space Smoothing}
To achieve this goal, we first introduce a probability density model for trajectories through a specific state-action pair. Let $\pi(a\vert s)$ denote a behavioral policy and the trajectory distribution $p_{\pi}$ induced by $\pi$ be expressed as:
\begin{equation}
\label{eq:p_pi}
  p_\pi(\tau) = \rho_0(s_0)\prod_{t=0}^{\infty}p(s_{t+1}\vert s_t, a_t)\pi(a_t\vert s_t),
\end{equation}
where $\rho_0(s_0)$ is is the distribution of the initial state $s_0$.

Given a state's information subset $x$, we can sample trajectories containing $x$ according to the distribution $p_\pi(\tau)$. Mathematically, suppose $x = f(s)$, then $x = s$ when $f$ is the identified map, and $x \neq s$ when $f$ is a general feature extraction function and not the identified map. In our experiments, $x$ is the coordinate $c$ of the center of mass (CoM). We design the following probability density model for trajectories including $x$:
\begin{equation}
  \label{eq:x_traj_distribution}
  p_\pi(\tau\vert x) := \frac{p_\pi(\tau)\mathbb{I}[x\in f(\tau)]}
  {\int_{\upsilon} p_\pi(\upsilon)\mathbb{I}[x\in f(\upsilon)]\mathrm{d}\upsilon},
\end{equation}
where $f(\tau) := \{f(s_0), f(s_1), \cdots\}$, $\mathbb{I}$ is the indicator function. $\tau$ is the current target trajectory and $\upsilon$ is an arbitrary trajectory, all sampled from $p_\pi$. In this manner, $p_{\pi}(\tau\vert x)$ is a conditional probability density function of trajectories through $x$ and describes the probability of $\tau$ among all trajectories through $x$.

\section{Proposed Approach}\label{sec:approach}
This section presents smooth policy optimization (POSG), a computationally efficient framework to achieve approximate long-term credit assignments. The key insight is that we propose to exploit state-only demonstration trajectories to measure the relative importance of current state-action pairs. Our approach fuses the distributional information of demonstrations with the return information of trajectories.

\subsection{State-Only Trajectory Distance}
In this paper, we treat trajectories as deterministic policy distributions and use the maximum mean discrepancy to empirically measure the difference (or similarity) between trajectories~\cite{gretton2006kernel, Gretton2012OptimalKC,masood2019diversity}. Given a current trajectory $\tau$ and a state-only demonstration trajectory $\tau_E$, the nonparametric test statistic, the maximum mean discrepancy, can be used to measure the distance between them:
\begin{equation}
    \label{equ:MMD_traj}
    \begin{aligned}
        \mathrm{MMD}^2(\tau,\tau_E,\mathcal{H}) = &\underset{o, o^\prime\sim\rho_\tau}{\mathbb{E}}\left[k\left(o, o^\prime\right)\right]\\
        &- 2\underset{\begin{subarray}{c}o\sim\rho_\tau\\ o_E\sim\rho_{E}\end{subarray}}{\mathbb{E}}\left[k(o,o_E)\right]\\
        &+ \underset{o, o_E^\prime\sim\rho_{E}}{\mathbb{E}}\left[k(o, o_E^\prime)\right],
    \end{aligned} 
\end{equation}
where $k(\cdot,\cdot)$ is the kernel of a reproducing kernel Hilbert space $\mathcal{H}$, $\rho_\tau(\cdot)$ and $\rho_E(\cdot)$ are the state visitation distributions of $\tau$ and $\tau_E$, and $o, o^\prime$ and $o_E$, $o_E^\prime$ are observations sampled from $\rho_\tau$ and $\rho_E$, respectively. In practice, the function $k(\cdot, \cdot)$ in Eq.~\eqref{equ:MMD_traj} is often defined as:
\begin{equation}
k(x,y) = K\left(g(x), g(y)\right).
\end{equation}

Formally, the function $g$ gives us the feasibility that adapts the focus of the ${\rm MMD}$ metric for different aspects to different downstream tasks. In our experiments, we define the distance $D(\tau, \mathcal{M}_E)$ from trajectory to the demonstration dataset $\mathcal{M}_E$ as follows:
\begin{equation}
    D(\tau, \mathcal{M}_E) = \min_{\tau_E\in\mathcal{M}_E}\mathrm{MMD}^2(\tau,\tau_E,\mathcal{H}).
\end{equation}

This definition of $D(\tau, \mathcal{M}_E)$ can only concern a relevant subset of the information in state observations. For example, we choose this information subset to be the coordinates $c=(x,y)$ of its center of mass (CoM) in the Key-Door-Treasure domain, i.e., the function $g$ maps a state observation $o$ to $c=(x,y)$. 

\subsection{Trajectory Importance Evaluation Mechanism}
\label{sec:guidance_reward_design}
This section defines the trajectory importance based on a novel trajectory-level distance measurement and the trajectory return. In this manner, this importance integrates the state distribution information of the demonstrations with the reward signals of the trajectories. We can then fuse these two pieces of information into the guidance reward by a smooth weighted average of the trajectory importance. Specifically, we introduce a trajectory weight function $\omega_{\pi}(\tau\vert \mathcal{M}_E)$ in Eq.~\eqref{eq:omega} based on the MMD distance metric in the trajectory space and the conditional distribution $p_\pi(\tau\vert x)$. Mathematically, an exponential function $e^{-kd}$ is adopted by $\omega_\pi(\cdot\vert\mathcal{M}_E)$ to smooth the weight of the trajectory, and $\omega_\pi(\cdot\vert\mathcal{M}_E)$ is expressed as follows:
\begin{equation}
  \label{eq:omega}
  \omega(\tau\vert \mathcal{M}_E) := \frac{e^{-kd(\tau)}}{\int_\upsilon p_\pi(\upsilon)e^{-kd(\upsilon)}\mathrm{d}\upsilon},
\end{equation}
where $d(\cdot)=D(\cdot, \mathcal{M}_E)$, $\tau$ and $\upsilon$ have the same meanings as those in Eq.~\eqref{eq:x_traj_distribution}, and $p_\pi(\cdot)$ is the trajectory distribution defined in Eq.~\eqref{eq:p_pi}. $k$ is a predefined positive constant to adjust the value of $\omega$. Specifically, by choosing the suitable value of $k$, we can smooth the value of $\omega$ and prevent its value corresponding to a certain trajectory from being too large or too small. Similar to $p_\pi(\cdot)$, $\omega(\cdot\vert\mathcal{M}_E)$ can be considered as a conditional probability density function over the trajectory space. $\omega(\cdot\vert\mathcal{M}_E)$ satisfies the conditions of probability distribution $\omega_{\pi}(\tau\vert\mathcal{M}_E)\ge 0, \forall\tau$ and $\int_\tau p_\pi(\tau)\omega(\tau\vert\mathcal{M}_E)\mathrm{d}\tau=1$. According to $\omega_{\pi}(\tau\vert \mathcal{M}_E)$, a trajectory closer to $\mathcal{M}_E$ is granted more weight by $\omega_\pi(\cdot\vert\mathcal{M}_E)$. Based on the trajectory weight function $\omega_{\pi}(\tau\vert \mathcal{M}_E)$, we can define the trajectory importance as follows:
\begin{equation}
    \label{eq:traj_importance}
    I(\tau) = \omega(\tau\vert\mathcal{M}_E)R_{j}(\tau),
\end{equation}
where the joint return $R_{j}$ is written as:
\begin{equation}
  \label{eq:R_j}
  R_{j}(\tau)=\alpha R(\tau)+\beta R(\tau_E).
\end{equation}

Here, $R(\cdot)$ denotes the return value of the trajectory. $\tau_E = \operatorname{arg\,min}_{\tau_E\in\mathcal{M}_E}\mathrm{MMD}^2(\tau,\tau_E,\mathcal{H})$ is the state-only demonstration in $\mathcal{M}_E$ closest to $\tau$ under the MMD distance metric, and $\mathcal{H}$ a reproducing kernel Hilbert space as described in Section~\ref{sec:backgroung}. The two constants $\alpha$ and $\beta$ in Eq.~\eqref{eq:R_j} are non-negative, and their linear sum is $1$, i.e., satisfies the condition $\alpha+\beta=1$. In this manner, the importance of the trajectory integrates the state distribution information of the demonstrations with the returns of the trajectories. 

\subsection{Smooth Guidance Reward Computation Technology}
In this section, we define the smooth guidance rewards for a state-action pair $(s, a)$ as follows:
\begin{equation}
  \label{eq:guidance_reward}
  r_{i}(s, a) = \int_\tau p_\pi(\tau\vert x)I(\tau) \mathrm{d}\tau.
\end{equation}
Here, $p_\pi(\cdot\vert x)$ is the probability density model for trajectories including $x$ defined in Eq.~\eqref{eq:x_traj_distribution}. Note that $x = f(s)$. The guidance reward $r_{i}(s, a)$ allocated to each state-action pair is the expectation of the product of $R_{j}(\tau)$ and $\omega(\tau\vert\mathcal{M}_E)$ under the conditional distribution $p_\pi(\tau\vert x)$. The joint return $R_{j}(\tau)$ is the weighted average sum of the environmental returns of $\tau$ and $\tau_E$. Moreover, we use $\omega(\tau\vert\mathcal{M}_E)$ to adjust its contribution to the guidance reward $r_{i}(s, a)$ according to Eq.~\eqref{eq:traj_importance}. Specifically, the closer the trajectory is from $\mathcal{M}_E$, the larger the weight of $\tau$.  Therefore, such a trajectory will contribute more to the smooth guidance reward. A new smoothed RL objective is obtained by inserting Eq.~\eqref{eq:guidance_reward} in the standard RL objective:
\begin{equation}
  \label{eq:smooth_objective}
  \tilde{\eta}(\pi_\theta) = \mathbb{E}_{\tau\sim\pi_\theta}\left[\sum_{t=0}^\infty\gamma^t r_{i}(s_t, a_t)\right].
\end{equation}

This is similar to the standard RL objective function in Eq.~\eqref{eq:rl_objective}, albeit with the guidance reward different from the environmental rewards. By maximizing this objective, the proposed method encourages the agent to revisit the hopeful region of the state space. In this manner, the state-action distribution of the current policy will gradually align with that of the demonstration trajectories. Hence, the agent can finally generate highly rewarded trajectories similar to the demonstration trajectories in the replay memory. Furthermore, experimental results indicate that determining the temporal structure of the sequential decision problem often results in high sample efficiency.

\begin{remark}
  Given a trajectory with the outcome, how can we determine the relevance of each state-action pair toward achieving this outcome in tasks with sparse or delayed rewards? This is the central topic that temporal credit assignment copes with. Our method can potentially perform long-term credit assignments. Inspired by LfD methods, we propose to realize this goal by using state-only demonstration data in sparse reward settings. Instead of learning complex models to reveal the relevance of action to the future result or adopting contribution analysis methods to achieve return redistribution, our method learns a guidance reward function to assist policy optimization. 
  
  Our method can be regarded as a simple redistribution method - the guidance reward for each state-action pair is obtained by calculating the weighted average of the trajectory returns based on the MMD distance between trajectories. In particular, this design of smooth guidance rewards considers the state distribution information of demonstrations and the trajectory returns simultaneously. More importantly, our method does not require obtaining trajectories with sparse rewards before performing credit assignments, which expands the application scope of CA approaches and provides a new idea to CA.
\end{remark}

\subsection{Algorithm Implementations for Discrete and Continuous Control Tasks}
In practice, without access to the true MDP dynamics, solving Eq.~\eqref{eq:guidance_reward} to obtain exact guidance rewards is infeasible. Fundamentally, $r_i(s, a)$ is the weighted expectation of joint rewards $R_j$ under the trajectory weight function $\omega(\tau\vert\mathcal{M}_E)$. Hence, we resort to the Monte Carlo (MC) method to calculate the estimation of the guidance reward: $r_{i}(s, a) = \mathbb{E}_{\tau\sim p_\pi(\tau\vert x)}\left[\omega(\tau\vert\mathcal{M}_E) R_{j}(\tau)\right]$. Let $\mathcal{B}$ denote a trajectory buffer generated by the current policy $\pi$ in the MDP. Then, the MC estimation of the guidance reward $r_{i}(s, a)$ can be given as:
\begin{equation}
  \label{eq:guidance_expectation}
  \hat{r}_{i}(s, a) = \frac{1}{N(x)}\sum_{\tau\in\mathcal{B}}\hat{\omega}(\tau\vert\mathcal{M}_E)R_{j}(\tau)\mathbb{I}[x\in f(\tau)]. 
\end{equation}
Here, $x=f(s)$, $N(x)$ is the number of trajectories in the buffer $\mathcal{B}$ satisfying $x\in f(\tau)$, and $\mathbb{I}$ is the indicator function. $\hat{\omega}(\tau\vert\mathcal{M}_E)$ is also estimated by the MC method and is written as: 
\begin{equation}
  \label{eq:omega_expectation}
  \hat{\omega}(\tau\vert \mathcal{M}_E) := \frac{e^{-kd(\tau)}}{\sum_{\upsilon\in\mathcal{B}}e^{-kd(\upsilon)}+\epsilon},
\end{equation}
where $d(\cdot)=D_{\rm MMD}(\cdot, \mathcal{M}_E)$ is introduced in Section~\ref{sec:backgroung}, $k$ and $\epsilon$ are predefined positive constants.

We can utilize the current policy $\pi$ to maintain buffer $\mathcal{B}$ and estimate the guidance rewards along with the computation of the policy gradient, which is practicable under the current RL framework. This guidance reward module can be embedded into any existing RL methods and helps them optimize a policy effectively in sparse reward settings. However, one potential issue of this simple approach is that it is challenging to collect adequate data, such that estimating the guidance reward converges in the high-dimensional state spaces. Perhaps more importantly, eagerly satisfying this requirement is typically needless and pointless.

For example, in a given 2D grid world shown in Fig.~\ref{fig:kdt_maze}, the goal position is at the up-right of the starting location of the agent. A handful of previous good demonstration trajectories are provided with the agent ahead of time. Then, the proposed method can be used to approximate the value of the guidance reward with trajectories collected in each iteration. The estimation may be imprecise; however, we only ensure that lower guidance rewards are received by state-action pairs on trajectories far away from $\mathcal{M}_E$. In this manner, the agent is encouraged to approach the regions where state-only demonstration trajectories of $\mathcal{M}_E$ stretch. Therefore, our method avoids additional sampling for achieving stationary credit assignments and does not incur high computational costs compared to standard RL algorithms.  Furthermore, we exploit past experiences gathered by the agent during training to update the previous good trajectory memory, which can help establish a more reliable guidance reward estimator.


\subsubsection{Applying to high-dimensional continuous spaces.}
When the state-action space of the environment the agent faces is high-dimensional and continuous, the proposed approach will encounter many thorny problems. Specifically, because of the continuity of the state-action space, the expectation in Eq.~\eqref{eq:guidance_reward} cannot be estimated accurately, even if we store the previous trajectories generated by the agent during the training process. Some necessary modifications must be made to scale the proposed approach to the complex environment. In Algorithm~\ref{algo:algo2}, we summarize PPO with POSG to obtain an algorithm example that can be applied to the high-dimensional continuous task. The smooth guidance reward computed with only a single trajectory is stored in the trajectory buffer. We then use this single smooth guidance reward to estimate the policy gradient for the policy parameter update. Mathematically, POSG degenerates to the Monte-Carlo estimation of the smooth guidance reward using a single trajectory rather than using many samples from the policy to estimate the expected credit assignment. This is not a noticeable problem in practice if the agent policy is parameterized by a deep neural network (DNN) since DNN tends to generalize well within the vicinity of the input data. Furthermore, experimental results indicate that Algorithm~\ref{algo:algo2} achieves competitive and superior performance in various high-dimensional tasks.

\section{Theoretical Analysis}
\label{sec:thm}
This section provides the theoretical foundation of POSG and illustrates its advantages in policy performance improvement. We analyze the performance improvement bound caused by smooth guidance rewards and derive a worst-case lower bound on the performance improvement. This result guarantees performance improvement and demonstrates the effectiveness of the proposed guidance rewards at a theoretical level. 

Consider that the trajectories in $\mathcal{M}_E$ are optimal, and $\pi_b$ represents a behavior policy implied by the trajectory data set $\mathcal{M}_E$. Then, due to the optimality of trajectories in $\mathcal{M}_E$, it can be viewed as an expert policy with high returns. Therefore, we expect the training strategy to behave like $\pi_b$ and obtain a higher return. To achieve this goal, we propose to design a smooth guidance reward function $r_i(s, a)$ to assist policy optimization and help the agent policy gradually converge to $\pi_b$ under the current paradigm of RL. The following lemma illustrates that this guidance reward is an available reward function to ensure $\pi_b$ is the optimal policy. It is worth mentioning that, in implementation, the trajectories in $\mathcal{M}_E$ may not be optimal at the beginning of policy optimization, and $\mathcal{M}_E$ can be updated with the highly rewarded trajectories generated during the training process.
\begin{restatable}{lemma}{OptimalLem}
  \label{lemma:1}
  Suppose $\pi_b$ is a policy implied by the replay memory $\mathcal{M}_E$ that contains all optimal trajectories, and $p(s^\prime|s) = \pi_b(a|s)P(s^\prime|s, a)$ is the state transition function consistent with $\mathcal{M}_E$. Let the discount factor $\lambda$ be 1. According to the definition of the smooth reward $r_i(s, a)$, if the current policy $\pi$ is expressed as:
  \begin{equation}
      \pi(a\vert s) = \left\{
      \begin{aligned}
           \pi_b(a\vert s),\quad  &\text{if} \  (s,a)\in\operatorname{supp}(\pi_b),\\
           0,\quad  &\text{else.}
      \end{aligned}
      \right.
  \end{equation}
  Then, $\pi$ is the optimal policy with the highest entropy under the smooth guidance reward $r_i(\cdot,\cdot)$ when the time horizon $T$ of MDP is finite.
\end{restatable}

An intuitive description is provided in the following remark to explain this assumption further.
\begin{remark}
  Lemma~\ref{lemma:1} gives a general explanation for the property of smooth guidance reward. The agent may rarely receive sufficient reward signals in the initial phase of policy optimization in environments with sparse rewards. Hence, it remains challenging for the agent to compute the gradient information accurately. According to this lemma, the smooth guidance reward function can provide dense reward feedback for policy optimization and encourage the agent to approach $\pi_b$ gradually, accelerating the agent's learning.
\end{remark}

Next, a deeper look at the proposed POSG algorithm is given to demonstrate its effectiveness in improving control performance. We quantitatively analyze the performance improvement bound in each iteration. The result obtained in this study indicates that POSG has a solid performance improvement guarantee compared to ordinary DRL algorithms. This conclusion relates the theoretical analysis of POSG with the results in previous works~\cite{kakade2002approximately, Schulman2015TrustRP, achiam2017constrained}. 

The performance bound connects the expected difference in the total return earned from environmental and guidance rewards to an average divergence between the new and old policies. Although the policy parameters are optimized with environmental and guidance rewards, the performance bounds are only related to environmental rewards. To this end, we assume that the behavior policy $\pi_b$ and smooth guidance reward satisfy the following requirement.
%
  %
%
\begin{restatable}{assumption}{ass}
    \label{ass:smooth_rewards}
    Let $A_e$ and $A_i$ be the advantage functions computed with environmental rewards $r_e$ and smooth guidance rewards $r_i$, respectively. Then, for any state-action pair $(s, a)$, there exists a constant $\lambda>0$, such that the following equation holds:
    \begin{equation}
    \label{eq:ass_A_i}
    \begin{aligned}
        A_i(s, a) &\ge \lambda A_e(s, a), \\
        \max_s \left\vert \mathbb{E}_{a \sim \pi^\prime} [A_i(s,a)] \right\vert &\le \lambda\max_s \left\vert \mathbb{E}_{a \sim \pi^\prime} [A_e(s,a)] \right\vert.
    \end{aligned}
    \end{equation}
\end{restatable}

\begin{remark}
    Temporal credit assignment refers to the problem of measuring the relevance of each state-action pair to future rewards. Various approaches have been proposed to design different models that estimate the influence of actions on future returns. POSG can be viewed as a simple credit assignment method - the influence of a state-action pair is a weighted average of the trajectory importance. Furthermore, as described in Lemma~\ref{lemma:1}, the smooth guidance reward is an available reward function to ensure $\pi_b$ is optimal. Therefore, we further consider that for each state-action pair, the advantage computed with the smooth guidance reward positively correlates with the environmental advantage to some extent. 
\end{remark}

Next, we consider the performance improvement bound in each iteration obtained by POSG. 
\begin{restatable}{theorem}{mainthm}[Performance Improvement Bound]
    \label{thm:bound}
    For any policies $\pi$ and $\pi^\prime$, let $r_i(s,a)$ be the smooth guidance reward defined in Eq.~\eqref{eq:guidance_reward}, and let $A_e(s,a)$ be the advantage function of the environmental reward, 
    \begin{equation*}
        \epsilon_{\pi^\prime} \doteq \max_s \left\vert \mathbb{E}_{a \sim \pi^\prime} [A_e(s,a)] \right\vert, 
    \end{equation*}
    \begin{equation*}
        T_{\pi}(\pi^\prime) \doteq \underset{\begin{subarray}{c} s\sim d^\pi\\ a\sim\pi\end{subarray}}{\mathbb{E}}\left[A_e(s,a)\right] \text{ and }
    \end{equation*}
    \begin{equation*}
        D_{\pi} (\pi') \doteq  \frac{T_{\pi} (\pi')}{1-\gamma} - \frac{2\gamma \epsilon_{\pi'}}{(1-\gamma)^2} \underset{s \sim d^{\pi}}{\mathbb{E}} \left[ D_{\rm TV} (\pi'\parallel\pi)[s] \right],
    \end{equation*}
    where $\gamma$ is the discount factor. $D_{\rm TV}(\pi'\parallel \pi)[s] = \frac{1}{2}\sum_a \left\vert \pi'(a\vert s) - \pi(a\vert s) \right\vert$ is used to represent the total variational divergence between two action distributions of $\pi$ and $\pi^\prime$ when the state is $s$. The following bounds hold:
    \begin{equation}
        \eta(\theta^\prime) - \eta(\theta) \geq (1+\lambda)D_{\pi} (\pi'), 
        \label{eq:bound1}
    \end{equation} 
    %
    where $\lambda$ is introduced in Assumption~\ref{ass:smooth_rewards}.
\end{restatable}
\begin{figure*}[ht]
  \centering
  \subfloat[]{
    \includegraphics[width=4.2cm]{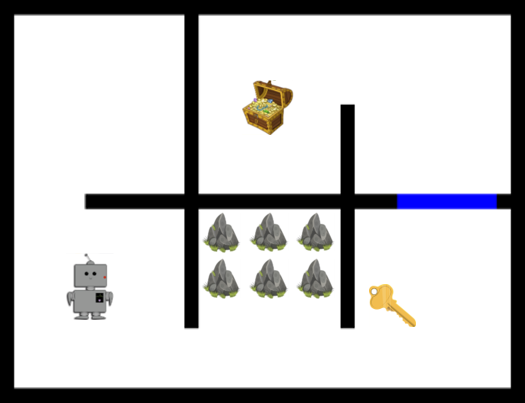}
    \label{fig:kdt_maze}
  }
  \subfloat[]{
    \includegraphics[width=3.3cm]{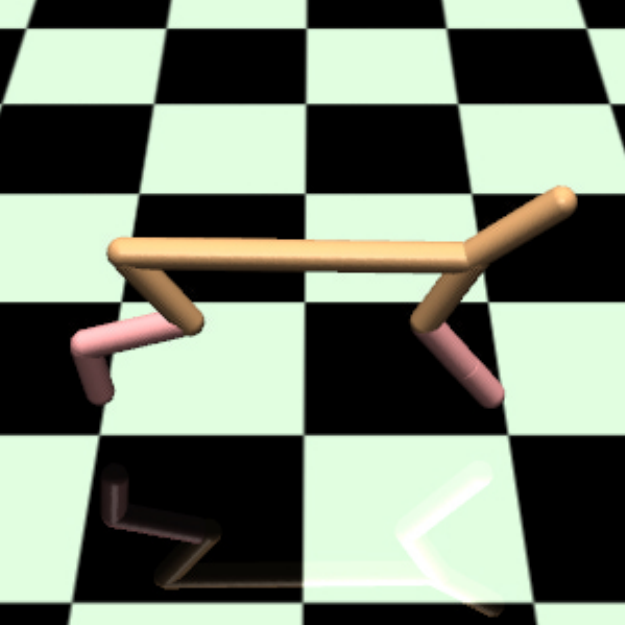}
    \label{fig:halfcheetah_envs}
  }
  \subfloat[]{
    \includegraphics[width=3.3cm]{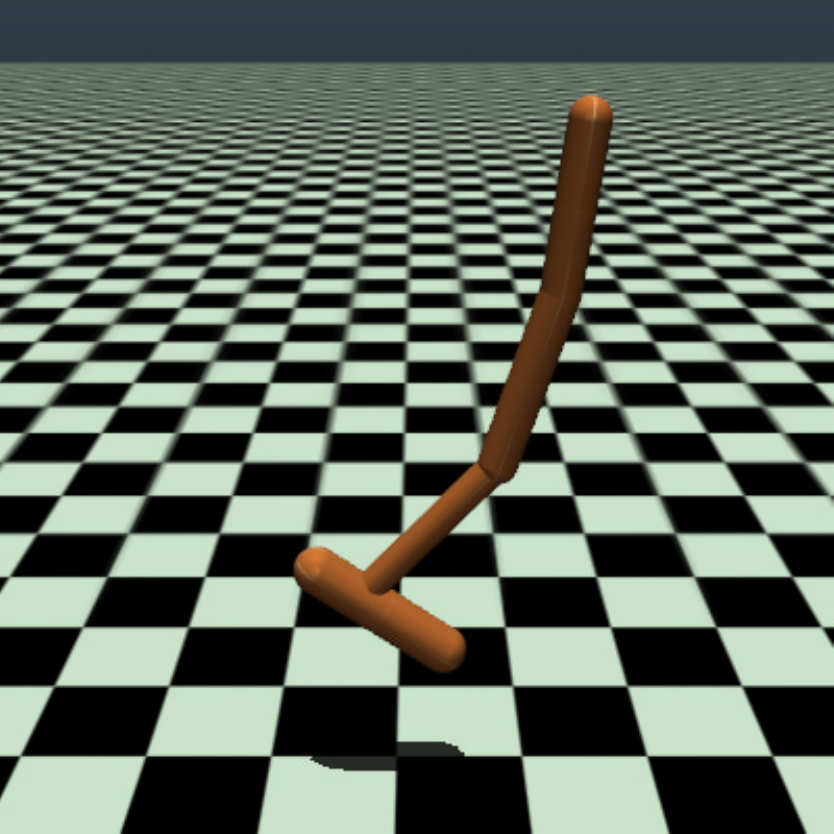}
    \label{fig:hopper_envs}
  }
  \subfloat[]{
    \includegraphics[width=4.2cm]{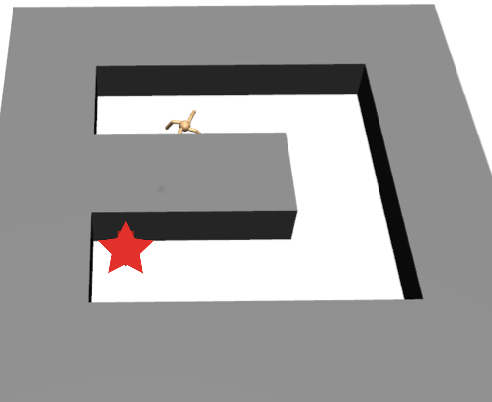}
    \label{fig:ant_maze}
  }
  \caption{
    A collection of environments that we used to evaluate POSG: (a) Key-Door-Treasure domain; (b) SparseHalfCheetah; (c) SparseHopper; (d) Ant Maze.
  }
  \label{fig:envs}
\end{figure*}

\begin{remark}
    Before proceeding, it is worth mentioning that Theorem~\ref{thm:bound} is similar to Corollary 1 of ~\cite{achiam2017constrained}. In the above theorem, a new performance improvement bound $(1+\lambda)D_{\pi,f} (\pi')$ is derived based on the environmental advantage function $A_e(s, a)$ and the smooth guidance reward function $r_i(s, a)$. This result illustrates that the smooth guidance reward can result in a broader performance improvement range with higher upper and lower bounds than that obtained only with the environmental reward. Consequently, the proposed method allows the agent to obtain a larger performance improvement in each policy optimization iteration.
\end{remark}

Note that the bound in Theorem~\ref{thm:bound} we have given is based on TV divergence between policies. For ease of calculation, we further connect the performance bound to KL divergence through Pinsker's inequality~\cite{csiszar2011information}: for two arbitrary distributions $p, q$, the TV and KL divergences satisfy the following in-equation: $D_{\rm TV}(p\parallel q) \le \sqrt{D_{\rm KL}(p\parallel q)/2}$. Combining this with Jensen's in-equation, we obtain
\begin{equation}
\begin{aligned}
    \underset{s \sim d^{\pi}}{\mathbb{E}}\left[D_{\rm TV}(\pi'\parallel\pi)[s]\right]
    &\le \underset{s \sim d^{\pi}}{\mathbb{E}}\left[\sqrt{\frac{1}{2} D_{\rm KL}(\pi'\parallel\pi)[s]}\right]\\
    &\le \sqrt{\frac{1}{2}\underset{s\sim d^{\pi}}{\mathbb{E}}\left[D_{\rm KL}(\pi'\parallel\pi)[s]\right]}.
\end{aligned}
\label{eq:tvkl}
\end{equation}

We obtain the following corollary by combining Eq.~\eqref{eq:tvkl} with the result of Theorem~\ref{thm:bound}.
\begin{restatable}{corollary}{WorstCaseBound}
\label{cor:eta_lower_bound}
    For any policies $\pi, \pi^\prime$ satisfying $\mathbb{E}_{s\sim d^\pi}[D_{\rm KL}(\pi^\prime\parallel\pi)[s]]\le\delta$. Combining Eq.~\eqref{eq:tvkl} with Eq.~\eqref{eq:bound1}, we have:
    \begin{equation}
        \eta(\pi^\prime) - \eta(\pi) \ge \frac{1+\lambda}{1-\gamma}\underset{\begin{subarray}{c}s\sim d^\pi \\ a\sim\pi^\prime\end{subarray}}{\mathbb{E}}\left[A_e(s,a)-\frac{\sqrt{2\delta} \gamma \epsilon_{\pi^\prime}}{1 -\gamma}\right],
        \label{eq:bound_cor1}
    \end{equation}
    where $\lambda$ is introduced in Assumption~\ref{ass:smooth_rewards}.
\end{restatable}

\begin{remark}
    The proposed guidance reward can be viewed as a smooth credit assignment over the trajectory space and characterize the influence and impact of each state-action pair the agent performs on future rewards. When the smooth guidance reward is used to update the current policy, the worst-case lower bound described in this corollary is $1+\lambda$ times the previous result obtained by~\cite{achiam2017constrained}. This result indicates that the proposed method generates a higher performance improvement boundary and provides a more stable performance guarantee in each iteration. Thus, Corollary~\ref{cor:eta_lower_bound} theoretically demonstrates the crucial advantage of POSG in obtaining higher control performances and accelerating training. 
\end{remark}

\section{Experimental Setup}
\label{sec:setup}
\subsection{Environments}
\textbf{Key-Door-Treasure Domain.} We first evaluated the performance of POSG in the Key-Door-Treasure task whose state-action space is discrete, as shown in Fig.~\ref{fig:kdt_maze}. The size of this grid-world environment is $26\times 36$. In each episode, the agent starts from the initial position in the bottom-left room of the Key-Door-Treasure domain. The maximum length of each episode is fixed, and an episode terminates immediately once the agent finds the treasure. The agent only receives a positive reward of 200 when it finds the treasure, and it cannot obtain any reward in other cases. At each time step, the agent is informed of its position information by the environment and chooses a possible action from the action space: \textit{move east, west, south,} and \textit{north}. To reach the location of the treasure, the agent is required to pick up the key (K) to open the door (D) and then travel through the room in the up-right corner to reach the treasure (T). 

\textbf{Locomotion Tasks from MuJoCo.} As shown in Figs.~\ref{fig:halfcheetah_envs} and~\ref{fig:hopper_envs}, we also demonstrated the effectiveness of POSG in several MuJoCo locomotion tasks with continuous state-action spaces. To investigate the potential limitations of POSG in more challenging learning tasks, we modified two classical MuJoCo agents, HalfCheetah and Hopper, and obtained two new agents named SparseHalfCheetah and SparseHopper: These agents yield a forward velocity reward only when the center-of-mass of the robot has already moved towards a certain direction for a threshold distance, and otherwise, the agent cannot obtain any positive reward. The threshold distance is 1 unit for SparseHopper and 10 units for SparseCheetah. At each time step, the agent observes the environment and performs an action sampled according to the policy. The agent receives an energy penalty caused by agent movement to adjust the torque applied to the robot joints. 

\textbf{Hard-Exploration Ant Maze.} 
To further investigate the performance of POSG, we evaluated this algorithm with the ant maze task introduced as the benchmark for RL by~\cite{duan2016benchmarking}, which is depicted in Fig.~\ref{fig:ant_maze}. The hard-exploration property of this task is caused by two factors: continuous locomotion control and navigation in the maze. More specifically, the Ant robot must first learn to walk smoothly and then struggle to reach the target position in the maze. The Ant robot is only rewarded a large positive bonus when it reaches the specified position of the maze. Meanwhile, the episode will end when the agent receives the reward. The state space of this environment consists of two main parts: the agent's internal joint angle information and task-specific attributes. The agent's joint angle information is only determined by the agent's internal state, and the task-specific attributes are mainly obtained from sensor readings, including the positions of walls and goals. 

\subsection{Neural Architectures and Hyper-parameters}
The neural networks were implemented for all tasks with fully connected networks with two layers of 64 hidden units. The discount factor for computing advantages is $0.99$ across all tasks. The policy neural networks were trained in the Key-Door-Treasure domain with a learning rate of $0.000022$. The maximum episode length was 240 steps. In the SparseCheetah task, the step size for neural network parameter optimization is $0.00009$, and the maximum length of each episode was 500. In the SparseHopper task, the learning rate for training policy neural networks was 0.0003, and the maximum length of each episode was 500. In the Ant-Maze environment, the maximum length of each episode was 500. The policy neural networks were optimized with a learning rate of $0.0001$. The maximum episode length was 750 in this task.

\subsection{Baseline Methods}
\label{subsubsec:baselines}
To investigate the benefits of using the smooth guidance rewards in the sparse reward setting, several different baseline methods were used for performance comparison in different tasks. We first compared our method with the state-of-the-art RL methods that learn from state-only demonstrations. To achieve this goal, we adopted GAIfO~\cite{torabi2018generative} as the baseline. However, this method is a state-only imitation learning method, and it does not interact with the environment to gain new experiences during training. For a fair comparison, we combined GAIfO with SIL~\cite{oh2018self} and obtained a new RL method GASIfO. Furthermore, we used several other state-of-the-art RL baselines, including SIL, PPO+D~\cite{libardi2021guided} and generative adversarial self-imitation learning (GASIL)~\cite{guo2018generative}. Unlike the standard SIL and GASIL framework in~\cite{guo2018generative}, our implemented GASIL could obtain the same demonstrations as POSG at the beginning of training. It is worth that these three methods can access the demonstrator's action information. We also ran proximal policy optimization (PPO)~\cite{schulman2017proximal} baseline to verify the reward sparsity. For all baseline approaches, we adopted the parameters that produced the best performance during the parameter search. All performance curves were obtained by averaging over ten separate runs with different random seeds, and the shaded regions represented the standard error over these ten runs.

\subsection{Acquisition of offline state-only demonstrations}
We used the PPO algorithm based on the default OpenAI Gym reward function to train expert policies for each task, shown in the figures as \textbf{expert}. Then, we used these policies to generate expert-level trajectories and only stored the observation and return information. We selected and saved policies learned during training as the medium-level policies named \textbf{medium} in the corresponding figures. We used these policies to generate medium-level trajectories and only preserved the observation and return information.
\begin{figure}[ht]
  \centering
  \subfloat[]{
    \includegraphics[width=4.3cm]{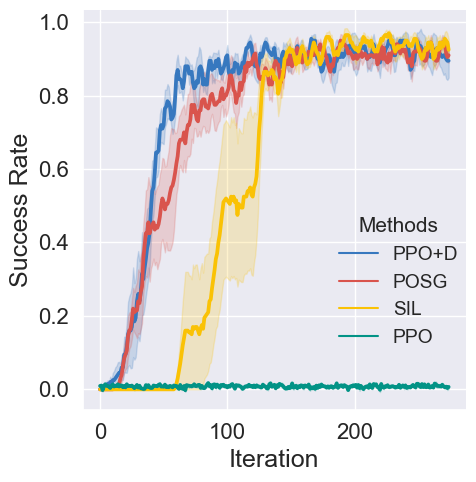}
    \label{fig:kdt_rate}
  }
  \subfloat[]{
    \includegraphics[width=4.3cm]{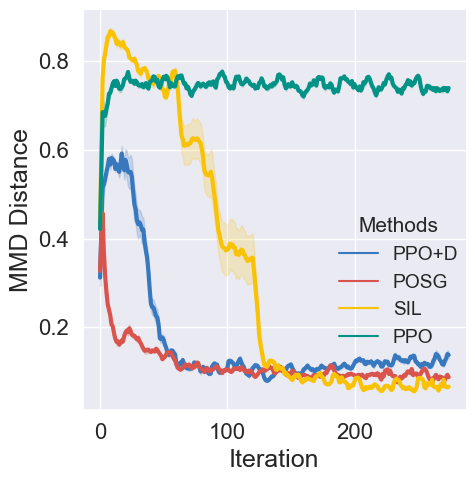}
    \label{fig:kdt_mmd}
  }
  \caption{
    (a) Success rate in the Key-Door-Treasure domain; (b) The changing trend of the MMD distance.
  }
  \label{fig:kdt_results}
\end{figure}

\section{Evaluation of Results}
\label{sec:experience}
In this section, to thoroughly assess the proposed POSG algorithm, we test it in three different classes of environments introduced in Section~\ref{sec:setup}, respectively, and compare the performance with other baseline algorithms. The experimental results demonstrate that POSG can achieve better performance than other methods. 
\begin{figure}[ht]
  \centering
  \subfloat[]{
    \includegraphics[width=3.6cm]{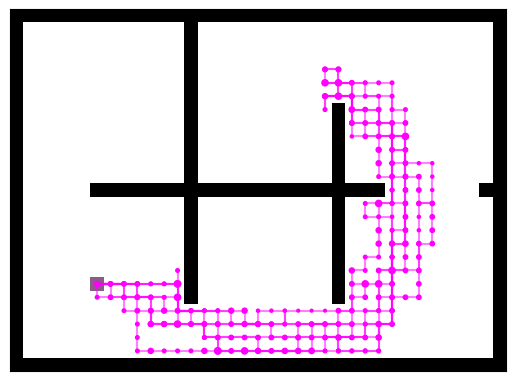}
    \label{fig:demo_trajs}
  }
  \subfloat[]{
    \includegraphics[width=3.6cm]{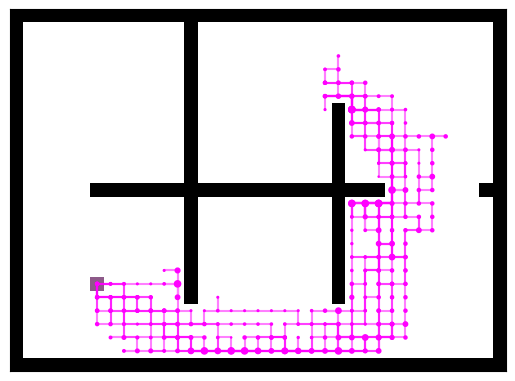}
    \label{fig:posg_trajs}
  }
  \caption{
    (a) State-action visitation graph of demonstrations; (b) State-action visitation graph of the POSG learned policy.
  }
\end{figure}

\subsection{Evaluations in the Key-Door-Treasure domain}
\subsubsection{Performance Comparison with Baseline Methods}
As shown in Fig.~\ref{fig:kdt_rate}, in the key-door-treasure domain, the sequential dependency property results in a low chance of completing all three tasks one by one: picking up the key, opening the door, and obtaining the treasure. We found that the PPO agent tends to get stuck at a sub-optimal policy that only picks up the key in a long time. It is worth noting that, for fairness, the SIL agent is provided with the same demonstration data as POSG at the beginning of policy optimization. The strong baselines, SIL and PPO+D, learn faster than PPO because these agents can obtain extra bonuses from the demonstration data. Specifically, SIL imitates and reproduces past good trajectories by storing experiences in the replay buffer and introducing a novel loss function. The PPO+D algorithm encourages agents to reach the goal quickly by replaying past good experiences. Interestingly, POSG can learn faster and is more stable than SIL, and POSG achieves competitive results against PPO+D in this sparse setting. This result indicates that our method can exploit the demonstrations better and obtain more useful and comprehensive information from these demonstrated trajectories since the smooth guidance reward design considers the demonstration's sparse-reward return and distribution.

We further investigated the changing trend of the MMD distance for all methods during the training process. The changing curves of the MMD distance are depicted in Fig.~\ref{fig:kdt_mmd}. As shown in Fig.~\ref{fig:kdt_mmd}, the value of the MMD distance decreases to zero at the end of the training for the POSG, SIL, and PPO+D algorithms. This result indicates that these three approaches acquire the optimal behaviors that a demonstrator recommends. Furthermore, POSG's MMD distance reduction rate is the highest among all methods. This result suggests that POSG avoids excessive meaningless exploration, and the smooth guidance reward provides more efficient policy optimization directions. Meanwhile, we draw the state-action visitation graphs of demonstrations and POSG learned policy in Figs.~\ref{fig:demo_trajs} and~\ref{fig:posg_trajs}, respectively. The comparative result demonstrated that the POSG agent imitates experts' behavior almost perfectly. 

\subsubsection{Visualization Comparisons for Exploration}
To illustrate how POSG explores the state space and achieves superior performance compared to other baseline methods, we plotted the state-visitation counts of all algorithms, depicted in Fig.~\ref{fig:state_visitation}. This figure compares the exploration regions of different agents in the same training phase. More specifically, the figures are obtained in the 36th iteration of policy optimization.
\begin{figure*}[!ht]
  \centering
  \subfloat[]{
    \includegraphics[width=3.6cm]{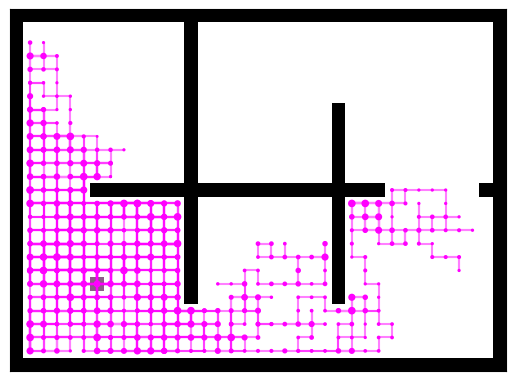}
    \label{fig:PPO}
  }
  \subfloat[]{
    \includegraphics[width=3.6cm]{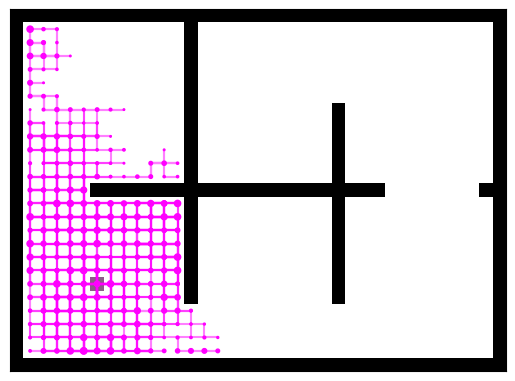}
    \label{fig:PPO+SIL}
  }
  \subfloat[]{
    \includegraphics[width=3.6cm]{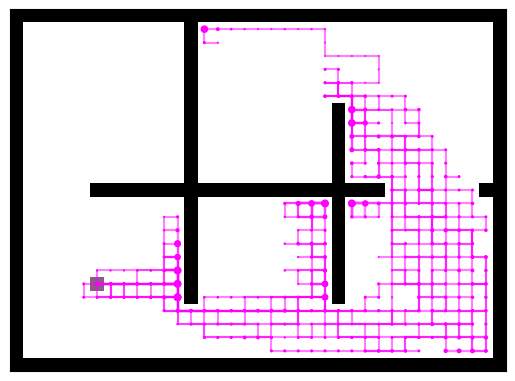}
    \label{fig:PPO+D}
  }
  \subfloat[]{
    \includegraphics[width=3.6cm]{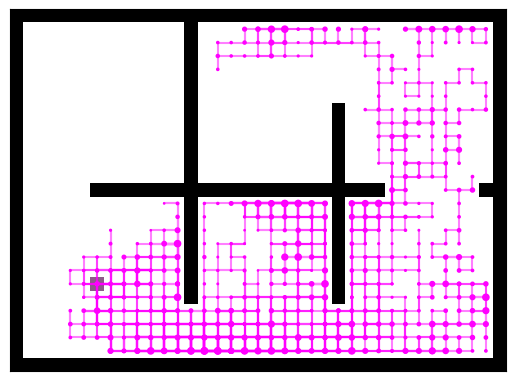}
    \label{fig:POSG}
  }
  \caption{
    The state-action visitation graphs of all algorithms: (a) PPO; (b) SIL; (c) PPO+D; (d) POSG.
  }
  \label{fig:state_visitation}
\end{figure*}

As shown in Fig.~\ref{fig:state_visitation}, the POSG and PPO+D agents can travel further from the initial position, while other agents can only hover around the initial point. Furthermore, the exploration region of the POSG agent is wider than that of the PPO+D agent. It turns out that POSG agents are more focused on the correct direction of exploration, and the exploration efficiency of POSG is higher than that of other baseline methods. This performance difference is caused by the integrated information utilization of the sparse-reward demonstrations, especially the position information. By merging the position information with the return signals, the proposed POSG method orients the agent's policy to softly approach the state-action distribution of sparse-reward demonstrations. Moreover, the return information ensures policy optimization efficiency by considering the temporal structure of MDP. 

\subsection{Performance on MuJoCo Locomotion Tasks}
\subsubsection{Performance Comparisons with Baseline Methods}
This experiment provided sparse-reward demonstrations with agents, and each trajectory contained only one return scalar to suggest its merit. The dense OpenAI Gym reward signals were removed. We used the default OpenAI Gym reward function to measure the quality of actions learned with POSG and other baseline approaches and computed the ground truth returns with this default OpenAI Gym reward function. We compared the ground truth returns of different methods to illustrate the superiority of POSG. 
\begin{figure}[ht]
\centering
\subfloat[]{
  \includegraphics[width=4.3cm]{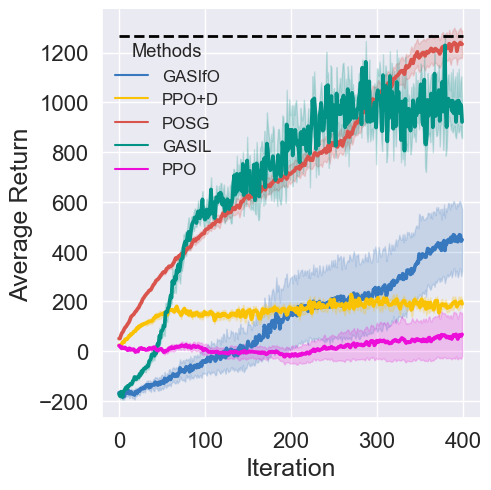}
  \label{fig:cheetah_return}
}
\subfloat[]{
  \includegraphics[width=4.3cm]{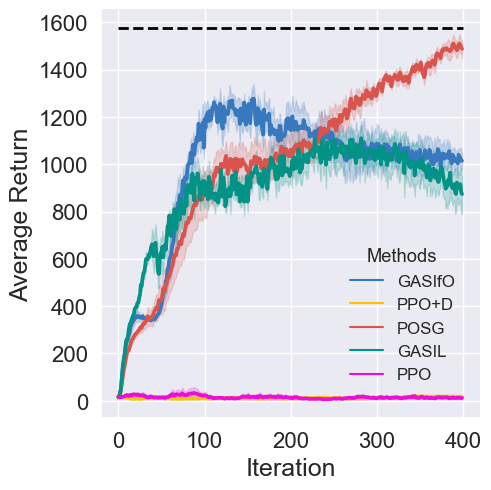}
  \label{fig:hopper_return}
}

\caption{
  (a) Learning curves of average return on the SparseHalfCheetah task; (b) Learning curves of average return on the SparseHopper task. 
}
\label{fig:locomotion_results}
\end{figure}

Figs.~\ref{fig:cheetah_return} and~\ref{fig:hopper_return} show that POSG outperforms other baseline approaches in the locomotion control tasks. Specifically, POSG can converge faster than different algorithms during policy optimization and achieve a higher final return. In contrast, PPO+D policy gradient computation only relies on the advantage information of demonstrations and thus cannot perform effective policy optimization with sparse-reward demonstrations. This causes the performance difference between PPO+D and POSG, and even the average return of PPO+D on the SparseHopper task no longer increases from the beginning of training. Since GASIL and GASIfO ignore reward signals of demonstrations and solely care about the distribution difference between the current policy and the sparse-reward demonstration data, they cannot obtain competitive results compared to POSG. The performance difference between GASIL and GASIfO in Fig.~\ref{fig:cheetah_return} demonstrates the importance of action information for learning. These results show that POSG achieves simple and efficient credit assignments with sparse-reward demonstrations. POSG avoids information bottlenecks by integrating the sparse-reward demonstrations' distribution information with the relevant trajectories' returns.
\begin{figure}[ht]
\centering
\includegraphics[width=6cm]{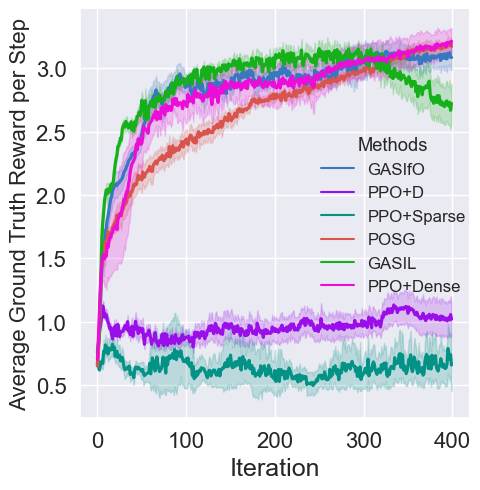}
\caption{
   Learning curves of average ground truth reward per step on the SparseHopper task. The dotted line represents the return value of the demonstration trajectories under the ground truth reward function. 
}
\label{fig:reward_learning}
\end{figure}

\subsubsection{Evaluation of Ground-Truth Reward Learning}
A deeper look into the experimental results is provided to further verify the POSG's effectiveness in learning near-optimal behaviors. In this experiment, for comparison, we trained the PPO algorithm with ground-truth environmental rewards that are the default setting of OpenAI Gym~\cite{openaigym}. We then calculated the average default reward of each step. Meanwhile, we deployed the experiment with POSG and other baseline methods in Fig.~\ref{fig:hopper_return} in the sparse reward version of the same task. The default OpenAI Gym's reward function was used to measure the quality of actions learned with these approaches, and then the average ground-truth reward for each step was computed. We compared the learning curves of the average ground-truth rewards for different algorithms. The experimental results are shown in Fig.~\ref{fig:reward_learning}. 
\begin{figure}[ht]
\centering
\subfloat[]{
  \includegraphics[width=8.5cm]{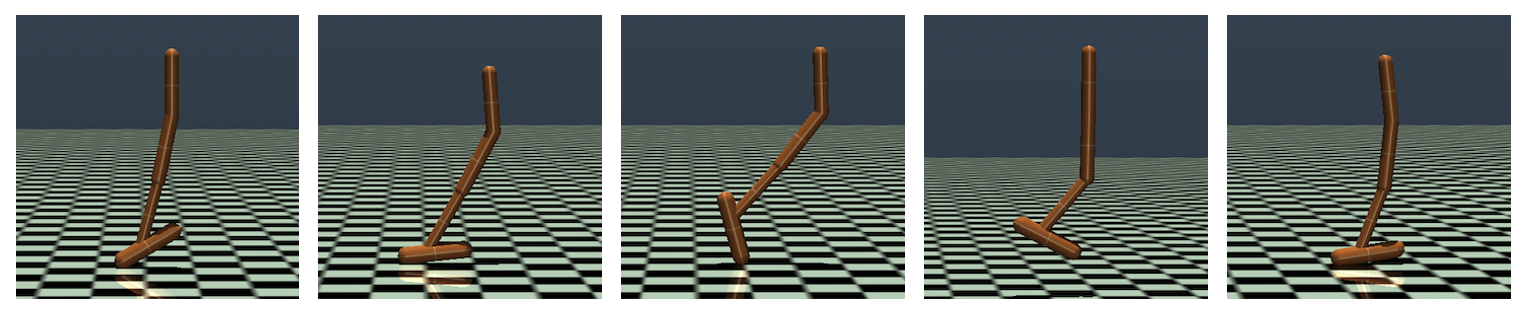}
  \label{fig:cheetah_traj}
}

\subfloat[]{
  \includegraphics[width=8.5cm]{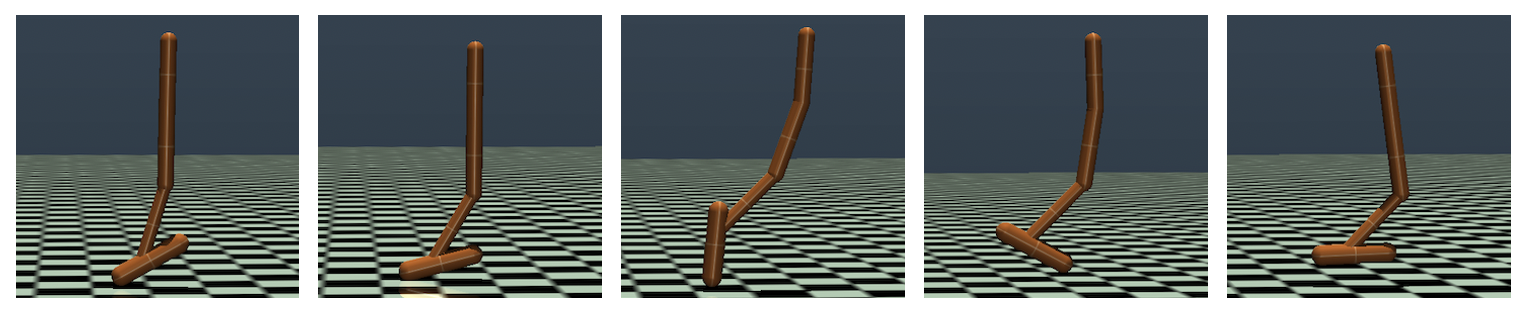}
  \label{fig:hopper_traj}
}
\caption{
  (a) Agent trained by POSG; (b) Agent trained by PPO with the ground-truth reward function. 
}
\label{fig:trajs}
\end{figure}

According to Fig.~\ref{fig:reward_learning}, the POSG average ground-truth reward dramatically increases at the beginning of learning, similar to PPO's learning trend in the default OpenAI Gym reward setting. At the same time, its value achieves an incremental improvement in the process of policy optimization. Moreover, the average ground-truth reward of POSG gradually approximates that of PPO in the default reward setting at the end of training. This result indicates that the smooth guidance reward can help the agent learn near-optimal behaviors that are rewarded a high score by the default OpenAI Gym reward function. For a more intuitive comparison, Fig.~\ref{fig:trajs} shows an example of two policies learned using POSG with sparse rewards and PPO with dense default rewards. Interestingly, the POSG agent has learned gaits similar to those of PPO. This fact further demonstrates the effectiveness of POSG in learning a near-optimal policy in tasks with sparse rewards.
\begin{figure}[ht]
\centering
\subfloat[]{
  \includegraphics[width=4.3cm]{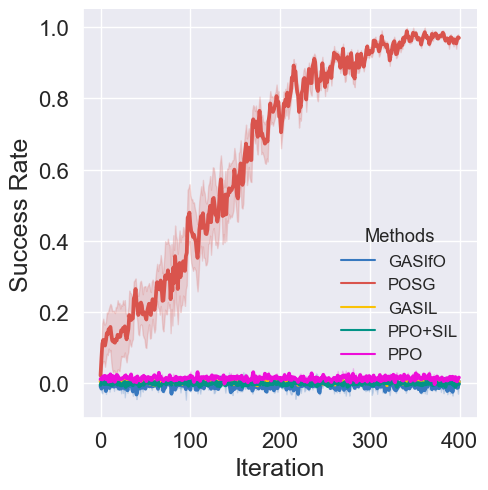}
  \label{fig:ant_rate}
}
\subfloat[]{
  \includegraphics[width=4.3cm]{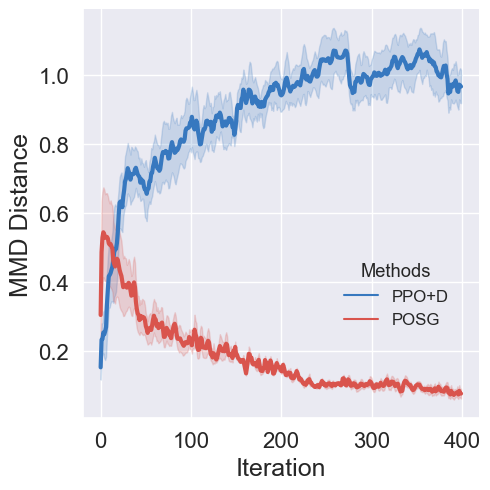}
  \label{fig:ant_mmd}
}
\caption{(a) Learning curves of success rate on the Ant Maze task; (b) Learning curves of average ground truth reward.}
\label{fig:ant_results}
\end{figure}

\subsection{Results in the Ant Maze}
To investigate how effective POSG is across various environments, we evaluated our proposed POSG algorithm on the hard-exploration Ant Maze task. It turns out that our approach significantly outperforms other baseline methods in terms of control performance. As shown in Fig.~\ref{fig:ant_rate}, only our method can reach the goal and learn the optimal policy in the hard-exploration ant maze. Other baseline methods cannot find the sparse reward even if the demonstrations are provided. This result demonstrates that our POSG can better exploit sparse-reward demonstrations by merging the demonstrations' distribution information and reward signals compared with other RLfD methods. Fig.~\ref{fig:ant_mmd} shows the changing curves of the MMD distance in the Ant Maze task. Our POSG method can efficiently decrease the distance between the training policy and sparse-reward demonstrations. Hence, POSG learns expert-like behaviors with demonstrations quickly. In contrast, the PPO+D agent cannot effectively learn the behaviors of demonstrations. Due to the sparsity of the reward function, the agent's policy often deteriorates during the policy optimization process, and the agent always falls near the initial point at the end of training.
\begin{figure}[!ht]
\centering
\subfloat[]{
  \includegraphics[width=4.3cm]{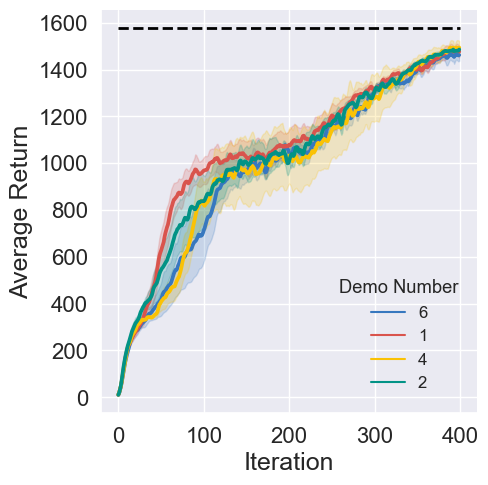}
  \label{fig:ablation_n_ret}
}
\subfloat[]{
  \includegraphics[width=4.3cm]{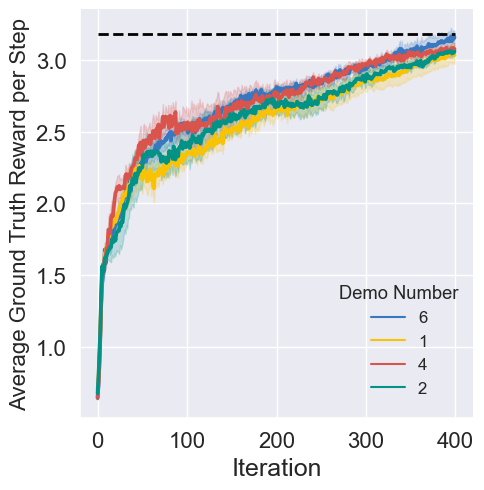}
  \label{fig:ablation_n_rew}
}
\caption{(a) Learning curves of average return; (b) Learning curves of average ground truth reward.}
\label{fig:ablation_number}
\end{figure}

\subsection{Ablation Analyses}
The experimental results described in the previous section indicate that POSG consistently outperforms other baseline approaches on several challenging tasks. We are now interested in whether these advantages still hold when changing the number and quality of sparse-reward demonstrations. We will compare the POSG performance on demonstrations with different amounts and quality to illustrate their impact on the performance results.
\begin{figure}[!ht]
\centering
\subfloat[]{
  \includegraphics[width=4.3cm]{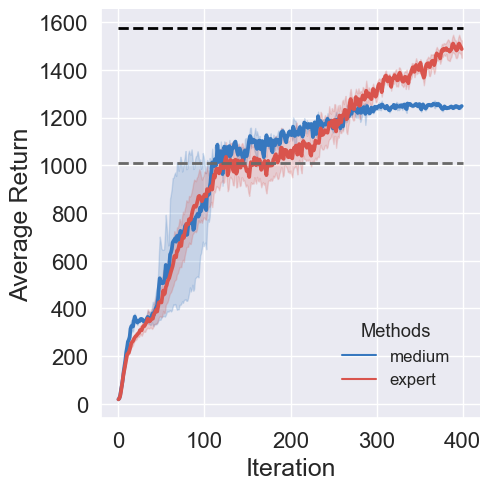}
  \label{fig:ablation_q_ret}
}
\subfloat[]{
  \includegraphics[width=4.3cm]{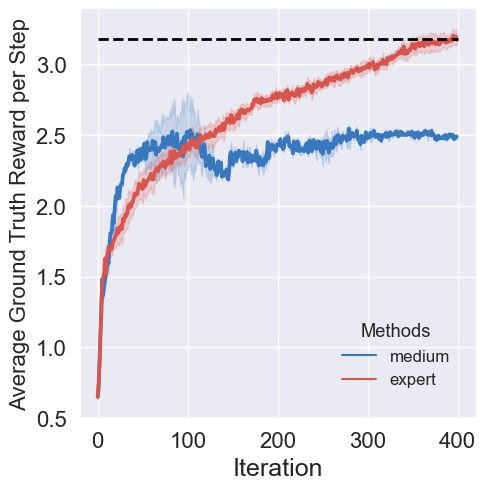}
  \label{fig:ablation_q_rew}
}
\caption{(a) Learning curves of average return; (b) Learning curves of average ground truth reward.}
\label{fig:ablation_quality}
\end{figure}

\subsubsection{Demonstrations with different amounts} 
\label{sec:abla_amount}
We selected five demonstration datasets with different amounts for comparison on the SparseHopper task. Notice that the POSG experiment is conducted with six trajectories of 500 state-action pairs as demonstrations. The black dashed line in Fig.~\ref{fig:ablation_n_ret} represents the expert-level performance. In Fig.~\ref{fig:ablation_n_rew}, The $y$-value corresponding to the dashed line is the value of the average reward per step learned by the PPO algorithm under the ground-truth reward setting. Fig.~\ref{fig:ablation_number} presents the result of this ablation experiment. Our POSG algorithm produces nearly identical results with different numbers of sparse-reward demonstrations. This result suggests that our POSG algorithm can achieve impressive performance using only a few demonstrations, even a single trajectory. Hence, our method decreases the requirement of demonstration quantity, which facilitates the application of the algorithm in real-world scenarios. 

\subsubsection{Demonstrations with different qualities} 
We obtained the demonstration datasets with different qualities by generating and storing the trajectory data in the different training phases. The corresponding result is shown in Fig.~\ref{fig:ablation_quality}. In Fig.~\ref{fig:ablation_q_ret}, the black dashed line represents the expert-level performance, and the gray dashed line represents the medium-level performance. The dashed line in Fig.~\ref{fig:ablation_q_rew} has the same meaning as that in Fig.~\ref{fig:ablation_n_rew}. The average return of POSG increases as the proportion of high-quality data increases. This result indicates that the quality of demonstrations will significantly affect the performance of POSG and that high-quality demonstration data can contribute to policy optimization to some extent.

\section{Conclusion}
\label{sec:conclusion}
This article proposes a simple and efficient algorithm called Policy Optimization with Smooth Guidance (POSG) that leverages a small set of state-only demonstrations (where only state information is included in demonstrations) to make approximate and effective long-term credit assignments while efficiently facilitating exploration. The key idea is that the relative impact of state-action pairs can be indirectly estimated using offline state-only demonstrations. More specifically, we obtain the importance of a trajectory by considering the distributional distance between policies and the returns of the associated trajectories. Then, the smooth guidance reward is computed by smoothly averaging the trajectory importance over the trajectory space. We theoretically analyze the performance improvement caused by smooth guidance rewards and derive a new worst-case lower bound on the performance improvement. We conducted various experiments by benchmarking POSG with several state-of-the-art RL algorithms in four environments. Experimental results demonstrate POSG's superiority over other baseline approaches in terms of performance and convergence speed.

\section*{Acknowledgments}
This work was supported by the National Key R\&D Program of China (2022ZD0116401), Fundamental Research Funds for the Central Universities and the National Natural Science Foundation of China(Grant Nos. 62141605)


{\appendices
\section{Proof of Lemma~\ref{lemma:1}}

\OptimalLem*
\begin{proof}
    Consider the environments with sparse rewards, the agent will not receive any non-zero reward unless it reaches the specific goals. Hence, the return of any trajectory generated by the agent is given as:
    \begin{equation}
        R(\tau) = \left\{
        \begin{aligned}
            &r, \text{if the agent reaches the goal,}\\
            &0, \text{else,}
        \end{aligned}
        \right.
    \end{equation}
    where $r$ represents the sparse reward in the environment.
    
  According to the definition of the trajectory weight function in Eq.~\eqref{eq:omega} and the joint return function~\eqref{eq:R_j}, a current trajectory can obtain the maximal importance weight only when it has the smallest distance from the replay memory $\mathcal{P}$. The smooth guidance reward in Eq.~\eqref{eq:guidance_reward} is the expectation of the product of the importance weight and joint return value for trajectories through $(s, a)$. Therefore, the smooth guidance reward for each state-action pair $(s, a)$ is maximal only when all the trajectories through $(s, a)$ are closest to $\mathcal{P}$. Moreover, the value of $r_i(s, a)$ for the state-action pair in $\mathcal{P}$ is larger than that of the other state-action pair.
  
  Note that in Eq.~\eqref{eq:smooth_objective}, when $\gamma=1$, the smooth RL objective function can also be written in the following form:
  \begin{equation}
      \tilde{\eta}(\pi_\theta) = \mathbb{E}_{\tau\sim\pi_\theta}\left[\sum_{t=0}^T r_{i}(s_t, a_t)\right].
  \end{equation}
  where $\sum_{t=0}^T r_{i}(s_t, a_t)$ represents the smooth return value of $\tau$, and the length of all trajectories is always set to the constant $T$. Hence, according to the analysis of $r_i$, the objective function $\tilde{\eta}$ can obtain the maximum value only when $\pi$ produces the same trajectories as $\pi_b$. Furthermore, if we require $\pi$ to be the highest entropy, then the trajectory distribution of $\pi$ should be identical to that of $\pi_b$:
  \begin{equation}
      \pi(a\vert s) = \left\{
      \begin{aligned}
           \pi_b(a\vert s),\quad  &\text{if} \  (s,a)\in\operatorname{supp}(\pi_b),\\
           0,\quad  &\text{else.}
      \end{aligned}
      \right.
  \end{equation}
\end{proof}

\section{Proof of Theorem~\ref{thm:bound} and Corollary~\ref{cor:eta_lower_bound}}
The following lemmas are proved in~\cite{achiam2017constrained}, and we have excerpted them here. The detailed proof process can be found in the appendix of~\cite{achiam2017constrained}.

\begin{lemma} 
For any function $f: S\to\mathbb{R}$ and any policy $\pi$,
\begin{equation}
    (1-\gamma) \underset{s \sim \rho_0}{\mathbb{E}}\left[f(s)\right] + \underset{\begin{subarray}{c} s \sim d^{\pi} \\ a \sim \pi \\ s' \sim P\end{subarray}}{\mathbb{E}} \left[ \gamma f(s') \right] - \underset{s \sim d^{\pi}}{\mathbb{E}}\left[f(s)\right] = 0, 
\end{equation}
where $\gamma$ is the discount factor, $\rho_0$ is the starting state distribution, and $P$ is the transition probability function. 
\label{lem:2}
\end{lemma}
\begin{lemma}
  \label{lem:policybound0} 
  For any function $f: S \to \mathbb{R}$ and any policies $\pi'$ and $\pi$, define
  \begin{equation}
    T_{\pi,f} (\pi') \doteq \underset{\begin{subarray}{c} s \sim d^{\pi} \\ a \sim \pi \\ s' \sim P\end{subarray}}{\mathbb{E}} \left[ \left(\frac{\pi'(a\vert s)}{\pi(a\vert s)} - 1 \right) \left(R(s,a) + \gamma f(s') - f(s) \right)\right],\label{surrogate}
  \end{equation}
  and $\epsilon_f^{\pi'} \doteq \max_s \left\vert \mathbb{E}_{\begin{subarray}{c}a \sim \pi', s'\sim P\end{subarray}} [R(s,a) + \gamma f(s') - f(s)] \right\vert$. Consider the standard RL objective function $J(\pi_{\theta}) = \mathbb{E}_{\tau}\left[\sum_{t=0}^{\infty}\gamma^{t} R(s_{t},a_t)\right]$, the following bounds hold:
  \begin{align}
    J(\pi') - J(\pi) \geq \frac{1}{1-\gamma}\left(T_{\pi,f} (\pi') - 2\epsilon_f^{\pi'} D_{\rm TV} (d^{\pi'} \parallel d^{\pi})\right), \label{bound0} \\
    J(\pi') - J(\pi) \leq \frac{1}{1-\gamma}\left(T_{\pi,f} (\pi') + 2\epsilon_f^{\pi'} D_{\rm TV} (d^{\pi'} \parallel d^{\pi})\right), \label{bound0b}
  \end{align}
  where $D_{TV}$ is the total variational divergence. Furthermore, the bounds are tight (when $\pi' = \pi$, the LHS and RHS are identically zero). Here, $\gamma$ is the discount factor, and $d^{\pi}$ is the discounted future state distribution.
\end{lemma}

\begin{lemma}
  \label{lem:divergencebound}
  The divergence between discounted future state visitation distributions, $\|d^{\pi'} - d^{\pi}\|_1$, is bounded by an average divergence of the policies $\pi'$ and $\pi$:
  \begin{equation}
    \|d^{\pi'} - d^{\pi}\|_1 \leq \frac{2\gamma}{1-\gamma} \underset{s \sim d^{\pi}}{\mathbb{E}} \left[ D_{\rm TV} (\pi' \parallel \pi)[s]\right], 
  \end{equation}
  where $D_{\rm TV} (\pi'\parallel \pi)[s] = (1/2) \sum_a \vert\pi'(a\vert s) - \pi(a\vert s)\vert$ is the total variational divergence at $s$.  
\end{lemma}

\mainthm*

\begin{proof}
  In the proposed methods, we use environmental rewards and smooth guidance rewards to compute the policy gradient and update the policy parameters in a two-step optimization manner. We can then analyze each optimization step's policy performance improvement bound individually. 
  
  First, in the environmental reward optimization step, the environmental reward is used to compute the policy gradient and update the policy parameters. With the bounds from Lemma~\ref{lem:policybound0} and bound the divergence $D_{\rm TV}(d^{\pi^\prime}\parallel d^\pi)$ by Lemma~\ref{lem:divergencebound}, when $f$ is the value function $V_e$ calculated by the environmental reward, we can easily obtain the following result:
  \begin{equation}
      \eta_e(\theta^\prime) - \eta_e(\theta) \geq D_{\pi} (\pi'), 
      \label{eq:bound1_1}
  \end{equation} 
  
  Second, we leverage the proposed smooth guidance reward to perform a trust-region-based policy optimization. In this case, let $R(s, a)$ in Lemma~\ref{lem:policybound0} be the smooth guidance reward $r_i(s, a)$, and let $f$ be the value function $V_i$ computed from smooth guidance rewards and $A_i$ be the advantage function. Define 
  \begin{equation}
      \label{eq:epsilon_tilde}
      \tilde{\epsilon}_{\pi^\prime} \doteq \max_s \left\vert \mathbb{E}_{a \sim \pi^\prime} [A_i(s,a)] \right\vert, 
  \end{equation}
  \begin{equation}
      \label{eq:T_tilde}
      \tilde{T}_{\pi}(\pi^\prime) \doteq \underset{\begin{subarray}{c} s\sim d^\pi\\ a\sim\pi\end{subarray}}{\mathbb{E}}\left[A_i(s,a)\right] \text{ and }
  \end{equation}
  \begin{equation}
      \label{eq:D_tilde}
      \tilde{D}_{\pi} (\pi') \doteq  \frac{\tilde{T}_{\pi} (\pi')}{1-\gamma} - \frac{2\gamma \tilde{\epsilon}_{\pi'}}{(1-\gamma)^2} \underset{s \sim d^{\pi}}{\mathbb{E}} \left[ D_{\rm TV} (\pi'\parallel\pi)[s] \right],
  \end{equation}
  With the bounds from Lemma~\ref{lem:policybound0} and bound the divergence $D_{\rm TV}(d^{\pi^\prime}\parallel d^\pi)$ by Lemma~\ref{lem:divergencebound}, the following result can be easily obtained: 
  \begin{equation}
      \tilde{\eta}(\pi^\prime) - \tilde{\eta}(\pi) \geq \tilde{D}_{\pi}(\pi'), 
      \label{eq:bound1_2}
  \end{equation}

  According to Assumption~\ref{ass:smooth_rewards}, substituting Eq.~\eqref{eq:ass_A_i} into Eq.~\eqref{eq:D_tilde}, we then have the following equation hold:
  \begin{equation}
      \label{eq:relation_tilde_D_and_D}
      \tilde{D}_{\pi} (\pi') \ge \lambda D_{\pi} (\pi')
  \end{equation}
  and
  \begin{equation}
      \label{eq:tilde_eta_bound_D}
      \tilde{\eta}(\pi^\prime) - \tilde{\eta}(\pi) \geq \lambda D_{\pi} (\pi'),
  \end{equation}
  The two above equations indicate that the original RL objective can further obtain performance improvement with the update of the smooth guidance reward. Finally, combining Eq.~\eqref{eq:bound1_1} with Eq.~\eqref{eq:tilde_eta_bound_D}, we then have $\eta_{total} = \eta + \tilde{\eta}$ satisfy the following relationship:
  \begin{equation}
      \eta(\pi^\prime) - \eta(\pi) \geq (1+\lambda)D_{\pi} (\pi').
  \end{equation}
\end{proof}

\WorstCaseBound*
\begin{proof}
    The result is obtained by combining Pinsker's inequality in Eq~\eqref{eq:tvkl} with Theorem~\ref{thm:bound}.
\end{proof}

\section{POSG Algorithm Training Process in Discrete Spaces}
Algorithm~\ref{algo:algo2} describes POSG in environments with discrete action space. 

\renewcommand{\algorithmicrequire}{\textbf{Input:}}
\renewcommand{\algorithmicensure}{\textbf{Output:}}
\begin{algorithm}[htb]
\caption{POSG in Discrete Spaces}
\label{algo:algo2}
\begin{algorithmic}[1]
  \REQUIRE learning rate $\alpha$, offline state-only demonstration dataset $\mathcal{P}$, policy update frequency $K$
  \STATE Initialize policy parameters $\theta$
  \STATE \textsc{// Buffer that stores key-value pairs of a state-action tuple $(\boldsymbol{s},\boldsymbol{a})$ and a list of values of trajectory importance that include $(\boldsymbol{s},\boldsymbol{a})$}
  \STATE Initialize $\mathcal{M} \leftarrow\emptyset\ \forall (\boldsymbol{s},\boldsymbol{a})$ 
  \FOR{$\text{episode}\in\{0,\dots,T\}$}
  \STATE $\tau\leftarrow\emptyset$\quad\quad\quad\quad \textsc{// Store state-action pairs for current episode}
  \STATE{$R\leftarrow 0$}\quad\quad\quad\quad \textsc{// Accumulate rewards for current episode}
  \FOR{each step in $\{1,\dots,T\}$}
  \STATE Choose $\boldsymbol{a}$ from $\boldsymbol{s}$ using $\pi_\theta$
  \STATE take action $\boldsymbol{a}$ and observe $\boldsymbol{r}_e$ and $\boldsymbol{s}^\prime$
  \STATE $\tau\leftarrow\tau\cup\{(\boldsymbol{s}, \boldsymbol{a})\}$; $R\leftarrow R + \boldsymbol{r}_e$
  \ENDFOR
  \STATE Compute the trajectory weight $\omega(\tau\vert\mathcal{M}_E)$ and the joint return $R_{j}(\tau)$
  \STATE Compute the trajectory importance $I(\tau) = \omega(\tau\vert\mathcal{M}_E)R_{j}(\tau)$
  \FOR{each $((\boldsymbol{s},\boldsymbol{a}))$ in $\tau_e$}
  \STATE $\boldsymbol{r}_i=\mathbb{E}_{I\sim\mathcal{M}(\boldsymbol{s},\boldsymbol{a})}[I]$\quad\quad \textsc{// Compute guidance rewards for each state-action pair}
  \ENDFOR
  \IF{$\text{episode}\, \%\, K == 0$}
  \STATE Compute $A_e$ using environmental rewards
  \STATE Compute $A_i$ using guidance rewards
  \STATE $\theta \leftarrow \theta + \alpha\nabla\eta_e(\theta)$ \quad\quad \textsc{// Update parameters using environmental rewards}
  \STATE $\theta \leftarrow \theta + \alpha\nabla\eta_i(\theta)$ \quad\quad\quad\quad \textsc{// Update parameters using guidance rewards}
  \ENDIF
  \ENDFOR
\end{algorithmic}
\end{algorithm}

\section{POSG Algorithm Training Process in High-Dimensional Continuous Spaces}
Algorithm~\ref{algo:algo1} describes POSG in environments with high-dimensional continuous action space. 

\renewcommand{\algorithmicrequire}{\textbf{Input:}}
\renewcommand{\algorithmicensure}{\textbf{Output:}}
\begin{algorithm}[htb]
\caption{POSG in High-dimensional Continuous Spaces}
\label{algo:algo1}
\begin{algorithmic}[1]
  \REQUIRE learning rate $\alpha$, offline state-only demonstration dataset $\mathcal{P}$, policy update frequency $K$
  \STATE Initialize policy parameters $\theta$ 
  \STATE \textsc{// Buffer that stores key-value pairs of a state-action tuple $(\boldsymbol{s},\boldsymbol{a})$ and a list of values of trajectory importance that include $(\boldsymbol{s},\boldsymbol{a})$}
  \STATE Initialize $\mathcal{M} \leftarrow\emptyset\ \forall (\boldsymbol{s},\boldsymbol{a})$ 
  \FOR{$\text{episode}\in\{0,\dots,T\}$}
  \STATE $\tau\leftarrow\emptyset$\quad\quad\quad\quad \textsc{// Store state-action pairs for current episode}
  \STATE{$R\leftarrow 0$}\quad\quad\quad\quad \textsc{// Accumulate rewards for current episode}
  \FOR{each step in $\{1,\dots,T\}$}
  \STATE Choose $\boldsymbol{a}$ from $\boldsymbol{s}$ using $\pi_\theta$
  \STATE take action $\boldsymbol{a}$ and observe $\boldsymbol{r}_e$ and $\boldsymbol{s}^\prime$
  \STATE $\tau\leftarrow\tau\cup\{(\boldsymbol{s}, \boldsymbol{a})\}$; $R\leftarrow R + \boldsymbol{r}_e$
  \ENDFOR
  \STATE Compute the trajectory weight $\omega(\tau\vert\mathcal{M}_E)$ and the joint return $R_{j}(\tau)$
  \STATE Compute the guidance reward $I(\tau) = \omega(\tau\vert\mathcal{M}_E)R_{j}(\tau)$
  \IF{$\text{episode}\, \%\, K == 0$}
  \STATE Compute $A_e$ using environmental rewards
  \STATE Compute $A_i$ using guidance rewards
  \STATE $\theta \leftarrow \theta + \alpha\nabla\eta_e(\theta)$ \quad\quad \textsc{// Update parameters using environmental rewards}
  \STATE $\theta \leftarrow \theta + \alpha\nabla\eta_i(\theta)$ \quad\quad\quad\quad \textsc{// Update parameters using guidance rewards}
  \ENDIF
  \ENDFOR
\end{algorithmic}
\end{algorithm}

 

\bibliographystyle{IEEEtran}
\bibliography{reference}

\end{document}